\documentclass[letterpaper, 10 pt, conference]{ieeeconf}

\IEEEoverridecommandlockouts
 
\usepackage{amsmath,graphicx,multicol}
\usepackage{amsfonts}
\usepackage{graphicx}
\usepackage{multicol}
\usepackage{color}
\usepackage{bbm}
\usepackage{cite}
\usepackage{amssymb}
\usepackage{algorithm}
\usepackage{algorithmic}
\usepackage{tabularx}
\usepackage[utf8]{inputenc}
\usepackage[english]{babel}

\usepackage{amsthm}
\usepackage{subcaption}
\usepackage{url}
\usepackage{mathtools}
\usepackage{comment}

\newtheorem{definition}{Definition}[]
\newtheorem{problem}{Problem}[]

\newtheorem{proposition}{Proposition}[]
\newtheorem{theorem}{Theorem}[]

\def\R{{\mathbb{R}}}
\def\Pr{{\mathit{Pr}}}
\def\mc{{\mathcal{MC}}}

\def\S{S}
\def\St{\tilde{S}}
\def\sinit{\mu_{init}}

\def\P{P}
\def\Pt{\tilde{P}}
\def\U{U}
\def\V{V}

\def\one{\mathbf{1}}
\def\gf{{\nabla f}}
\DeclareMathOperator*{\rank}{rank}
\def\path{\sigma}

\newcommand{\ts}{\textsuperscript}

\def\R{{\mathbb{R}}}
\def\N{{\mathbb{N}}}

\newcommand{\RNum}[1]{\uppercase\expandafter{\romannumeral #1\relax}}
\newcolumntype{C}{>{\centering\arraybackslash}b{\widthof{positions}}}
\newcolumntype{d}{D{.}{.}{-2}}

\title{\bf Identifying Sparse Low-Dimensional Structures in Markov Chains:\\A Nonnegative Matrix Factorization Approach}

\author{
Mahsa Ghasemi, Abolfazl Hashemi, Haris Vikalo, and Ufuk Topcu
\thanks{Mahsa Ghasemi, Abolfazl Hashemi and Haris Vikalo are with the Department of Electrical and Computer Engineering, and Ufuk Topcu is with the Department of Aerospace Engineering and Engineering Mechanics of University of Texas at Austin, Austin, TX 78712 USA.}%
\thanks{This work was supported in part by ONR grant \# N000141712623, and NSF grants \# 1652113 and \# 1809327.
}%
}

\begin{document}
    
\maketitle

\begin{abstract}
We consider the problem of learning low-dimensional representations for large-scale Markov chains. We formulate the task of representation learning as that of mapping the state space of the model to a low-dimensional state space, called the \textit{kernel space}. The kernel space contains a set of meta states which are desired to be representative of only a small subset of original states. To promote this  structural property, we constrain the number of nonzero entries of the mappings between the state space and the kernel space. By imposing the desired characteristics of the representation, we cast the problem as a constrained nonnegative matrix factorization. To compute the solution, we propose an efficient block coordinate gradient descent and theoretically analyze its convergence properties.
\end{abstract}

\section{Introduction}\label{sec:intro}
A variety of queries about stochastic systems boil down to study  of Markov chains and their properties. They have been widely used as a modeling tool in applications including control~\cite{kushner1971introduction}, machine learning~\cite{bishop2006pattern}, and computational biology~\cite{eddy1996hidden}. Moreover, they create the foundation for more complex probabilistic graphical models including hidden Markov models and Markov decision processes~\cite{koller2009probabilistic,ghasemi2019perception}.

In many practical settings, a system modeled as a Markov chain, has a large state space. For instance, fine discretization of a zero-input dynamical model with continuous space may lead to a Markov chain with a huge discrete state space. The fact that analyzing such large-scale models may be intractable has motivated significant research on model reduction algorithms. These algorithms attempt to create compressed abstractions that enable efficient downstream analysis without compromising the performance.

A key enabling factor in abstracting Markov chains is the existence of certain structural properties of the characterizing transition probabilities. For instance, the transition probabilities, captured by a stochastic matrix, may be low-rank or sparse. Therefore, one can exploit these structural properties to construct abstractions that accurately approximate the original model. 
The present work is motivated by the state aggregation framework for reducing the complexity of
reinforcement learning and control systems~\cite{abel2018state}. State aggregation schemes  attempt to group similar states into a small number of meta states, which are typically handpicked based on domain-specific knowledge~\cite{rogers1991aggregation,bertsekas1996neuro}, or based on a given similarity metric or feature function~\cite{tsitsiklis1996feature}.

In this paper, we propose an algorithm to find a surrogate representation that approximates the original Markov chain while having a low-dimensional state space. The proposed framework aims to learn a bidirectional mapping between the original high-dimensional state space and the low-dimensional state space, referred to as the kernel space. Additionally, to improve representativeness of the states in the kernel space, we constrain the mappings to be sparse. We model the task of learning the mappings and the kernel transition as a combinatorial optimization problem. Then, we relax this formulation and establish a sparsity-promoting constrained nonnegative matrix factorization problem.
In order to solve this factorization, we propose an efficient block coordinate gradient descent algorithm that starting from an initial guess, learns the bidirectional mappings and the kernel transition in an iterative fashion. We further prove that under certain conditions on the step sizes, the algorithm converges to a stationery point of the proposed optimization problem. We complement our methodology with extensive simulation results where we demonstrate efficacy of the proposed algorithm in terms of the quality of the low-dimensional representation as well as its computational cost.

\subsection{Related Work}\label{ssec:rel}

In the analysis of dynamical systems, different model-reduction techniques have been designed, such as approximating the transfer
operators~\cite{molgedey1994separation}, dynamic mode decomposition~\cite{schmid2010dynamic}, and data-driven approximations~\cite{klus2018data}.
In control theory and reinforcement learning, state abstraction has been widely studied as a way of reducing the complexity of computing the optimal controller or optimal value function~\cite{bertsekas1996neuro,ren2002state}.
Additionally, a related direction of research, called representation learning, tries to construct basis functions for representing high-dimensional value functions. Different Laplacian-based methods have been proposed to generate a surrogate for the exact transition operator~\cite{johns2007constructing,parr2007analyzing,petrik2007analysis}.

Matrix factorization is an optimization framework that decomposes a matrix into a product of two or more matrices~\cite{lee2001algorithms}. In general matrix factorization framework, as opposed to spectral decomposition, one can impose additional desired structural properties. Common structural properties are those of being low-rank or sparse that can be promoted by nuclear norm regularization and $\ell_1$-norm regularization, respectively. Furthermore, matrix factorization is typically amenable to efficient gradient descent solutions.
The problem of decomposing a matrix into a product of factors arises in different applications such as learning Markov models~\cite{cybenko2011learning} and bioinformatics~\cite{hashemi2018sparse}. What makes matrix factorization methods appealing is the fact that their solution complexity depends on the rank of the factors which is typically much smaller than the input matrix. 

Recovery of a low-rank probability transition matrix has been considered in~\cite{hsu2012spectral,huang2016recovering,li2018estimation,zhang2019spectral}. The majority of the proposed methods are based on spectral decomposition framework. 
In~\cite{duan2019state}, the authors use the notion of anchor states to enhance the interpretability of the meta states and further, develop a method of learning a soft-aggregation model from a set of system's trajectories.
In contrast, we propose an optimization formulation for the task of learning a low-dimensional representation of a known Markov chain. This formulation can easily induce different desired structural properties, such as sparsity, by imposing additional constraints. Furthermore, it enables an efficient solution that relies on block coordinate gradient descent.
\section{Problem Formulation}\label{sec:prob}
In this section, we provide the outline of the related concepts and definitions, and formally state the problem of learning representations for Markov chains.

\subsection{Preliminaries}

In this paper, we focus on time-homogeneous discrete-time Markov chains with finite states, as formally defined below.

\begin{definition}[Markov Chain]
	A Markov chain is a random process such that its evolution is characterized by a tuple $\mc = (\S, \sinit, \P)$, where $S$ is a finite set of states with cardinality $|S| = n$, $\sinit$ is an initial distribution over the states, and $\P: \S \times \S \to [0, 1] \subseteq \R$ is a probabilistic transition function such that for all $s \in \S$, ${\sum_{s'\in \S}\P(s,s')=1}$.
\end{definition}
A finite path in $\mc$ is a realization of a finite-length sequence $X_0,X_1,\ldots$ of states, denoted by $\path = x_0 x_1 x_2 \ldots$, such that $x_0$ is in the support of $\sinit$ and $\forall i \in \mathbb{Z}: \P(x_i,x_{i+1}) > 0$. Using the Markovian property, the probability of sampling $\path = x_0 x_1 x_2 \ldots x_T$ is
\begin{equation*}
\Pr(\path) = \sinit(x_0) \sum_{t=1}^{T} \P(x_{t-1},x_t),
\end{equation*}
and the probability of going from state $s_i$ at time step $t$ to state $s_j$, in $m$ time steps, is $\Pr(X_{t+m}=s_j | X_{t}=s_i) = p_{ij}^{(m)}$, where $p_{ij}^{(m)} = [P^m]_{ij}$ is an entry of the $m$-step transition matrix.

Runnenburg~\cite{runnenburg1966markov} introduced the notion of Markov chains with small rank as a type of dependence that is close to independence. Hoekstra~\cite{hoekstra1983markov} has further analyzed properties of Markov chains with small rank. Next, we provide the formal definition of the nonnegative rank of a Markov chain that will later motivate the proposed factorization.

\begin{definition}[Nonnegative Rank of Markov Chain]
	Let $\P$ denote the transition matrix of a Markov chain $\mc$. The nonnegative rank of $\mc$ is the smallest $k \in \N$ for which the following decomposition exists:
	\begin{equation}
		\Pr(X_{t+1} | X_t) = \sum_{l=1}^{k} f_l(X_t) g_l(X_{t+1}),
	\end{equation}
	where $f_1,f_2,\ldots,f_k$ and $g_1,g_2,\ldots,g_k$ are real-valued functions mapping $\S$ to $\R_+$.
\end{definition}
In particular, $f_1,f_2,\ldots,f_k$ denote the left Markov features and $g_1,g_2,\ldots,g_k$ denote the right Markov features. Without loss of generality, one can assume that the left and right Markov features are probability mass functions. Notice that if the nonnegative rank of $\mc$ is $k$, it holds that $\rank(\P) \le k$~\cite{hoekstra1983markov}.
In other words, the nonnegative rank of a Markov chain upperbounds the rank of its transition matrix.

The next proposition states a critical property of nonnegative rank of Markov chains that we will later exploit.
\begin{proposition}[Decomposition into Stochastic Matrices\cite{zhang2019spectral}]\label{prop:1}
	The nonnegative rank of a Markov chain is $k$ if and only if there exists $\U \in \R_+^{n \times k}$, $\Pt \in \R_+^{k \times k}$, and $\V \in \R_+^{k \times n}$ such that $\P = \U \Pt \V$, where $U$, $P$, and $\V$ are stochastic matrices, i.e., $\U \one = \one$, $\Pt \one = \one$, and $\V \one = \one$.
\end{proposition}
In Proposition~\ref{prop:1}, $\Pt$ resembles a surrogate lower-dimensional model for the transition matrix $\P$ if $\P$ admits a low nonnegative rank. We refer to the state space and the transitions in this abstract model as \textit{kernel space} and \textit{kernel transition}, respectively. 
In Proposition~\ref{prop:1}, $\U$ maps the states of the original Markov chain into the kernel space, $\Pt$ indicates the kernel transition, and $\V$ maps the kernel states back to the original states. 
In control and reinforcement learning, rows of $\U$ correspond to aggregation distributions while rows of $\V$ correspond to disaggregation distributions~\cite{bertsekas1995dynamic}.

The strength of the model abstraction setting introduced in Proposition~\ref{prop:1} is that it is independent of the downstream analysis to be performed on the Markov chain. Therefore, once such abstraction is found, it can be used for accelerating different types of analyses. For instance, in the next proposition, we show how the $m$-step transition matrix can be computed more efficiently by using the factorized model.
\begin{proposition}[Efficient $m$-Step Transition]\label{prop:2}
	Given a Markov chain $\mc = (\S, \sinit, \P)$, assume that a perfect low-rank decomposition of the transition matrix exists such that $\P = \U \Pt \V$, $\Pt \in \R_+^{k \times k}$. Let $K = \V \U \Pt$. Then, the $m$-step transition matrix of $\mc$ can be computed by
	\begin{equation*}
		\Pr(X_{t+m} | X_{t}) = \sum_{l_1=1}^{k} \sum_{l_2=1}^{k} 
		\U_{X_{t},l_1} [\Pt K^{m-1}]_{l_1 l_2} \V_{l_2,X_{t+m}}.
	\end{equation*}
\end{proposition}
Hence, one can reduce the complexity of computing the $m$-step transition matrix from $\mathcal{O} (mn^2)$ to $\mathcal{O} (mk^2)$. 

\subsection{Problem Statement}

In many applications, often the Markov chain model of a system has underlying structural properties, including possession of a low-rank or sparse transition function. Motivated by this fact, we seek an abstraction of a Markov chain in a low-dimensional kernel space. To that end, we need to find the mapping from the original state space to the state space of the abstracted (surrogate) model as well as the inverse mapping from the state space of the surrogate model to the original state space. Figure~\ref{fig:mapping} demonstrates a pictorial overview of the mapping between the spaces. Essentially, a transition in the original model can be represented through three steps: step 1 maps a state in the original model to meta states in the surrogate model; step 2 is a transition inside the surrogate model; and step 3 is a mapping from a meta state back to original states. Additionally, we would like the meta states to be representative of a small subset of states. This property means that each meta state should be connected to as few states as possible and hence, we impose that by looking for sparse mappings between the spaces. 

\begin{problem}
	Given a Markov chain $\mc = (\S, \sinit, \P)$, we aim to find a kernel space and kernel transition, denoted by $(\St, \Pt)$, with lower dimensionality, i.e., $|\St| \ll |\S|$. Further, we look for a sparse bidirectional mapping $(\U,\V)$ where $U$ represents the mapping from $\S$ to $\St$ while $\V$ represents the mapping from $\St$ to $\S$. The surrogate model $(\St, \Pt)$ along the bidirectional mapping $(\U,\V)$ must be such that the decomposition property $\P = \U \Pt \V$ holds.
\end{problem}

\section{Approach}\label{sec:alg}
Let $n = |\S|$ to be the size of the high-dimensional state space and $k$ denote the nonnegative rank of the $\mc$. Let $\mathcal{D}: \R^{n\times n}\times \R^{n\times n} \rightarrow \R_+$ denote a metric on the space of $n\times n$ matrices. As we discussed in Section \ref{sec:prob}, we seek to promote sparsity patterns in the rows of the bidirectional mapping $(\U,\V)$. Therefore, in order to find the factorization in Problem 1, we propose the following optimization task
\begin{equation}\label{p0pre}
\begin{aligned}
& \underset{\U \ge 0,\Pt \ge 0,\V \ge 0}{\text{min}}
\quad \mathcal{D}(\P,\U\Pt\V)\\
& \text{s.t.}\quad \sum_{j = 1}^k \U_{ij} = 1, \quad \|u_{i}\|_0 \leq s_i^{(u)}, \forall i \in [n], \\
&\qquad \sum_{j = 1}^k \Pt_{\ell j} = 1, \quad \forall \ell \in [k],\\
&\qquad \sum_{j = 1}^n \V_{\ell j} = 1,  \quad \|v_{\ell}\|_0 \leq s_\ell^{(v)}, \forall \ell \in [k],\\
\end{aligned}
\end{equation}
where $\|.\|_0$ is the so-called $\ell_0$-norm and returns the number of nonzero entries of its argument, and $\{s_i^{(u)}\}$ and $\{s_\ell^{(v)}\}$ are positive integers that determine the extent of the desired sparsity structure in the rows of $\U$ and $\V$. Notice that because of the  $\ell_0$-norm constraints, \eqref{p0pre} is a combinatorial optimization problem and generally NP-hard to solve. Therefore, we propose to relax these constraints by using the $\ell_1$-norm which is the convex envelope of the $\ell_0$-norm and is known to promote sparsity in the solution of an optimization problem. Following this idea and by specifying $\mathcal{D}(X,Y) = \frac{1}{2}\|X-Y\|^{2}_{F}$ as the metric, we consider the relaxed and regularized problem
\begin{equation}\label{p0}
\begin{aligned}
& \underset{\U \ge 0,\Pt \ge 0,\V \ge 0}{\text{min}}
\quad \frac{1}{2}\|\P-\U\Pt\V\|^{2}_{F}+\lambda_u\|\U\|_1+\lambda_v\|\V\|_1\\
& \text{s.t.}\quad  \U \one = \one, \quad \Pt \one = \one, \quad \V \one = \one.\\
\end{aligned}
\end{equation}
Here, $\lambda_u>0$ and $\lambda_v>0$ are the regularization parameters that determine the sparsity level of the rows of the bidirectional mapping matrices $\U$ and $\V$. Recall that sparsity of $\U$ means that each state is mapped to one (or few) meta state(s). On the other hand, sparsity of $\V$ asks for mapping of a meta state to a few states. This sparsity promoting terms ensure that the meta states are a good representative of their corresponding states. Additionally, the constraints of the optimization make sure that each of matrices are a stochastic matrix, i.e., can be interpreted as a transition matrix.

\begin{figure}[t]
	\centering
	\includegraphics[width=0.45\textwidth]{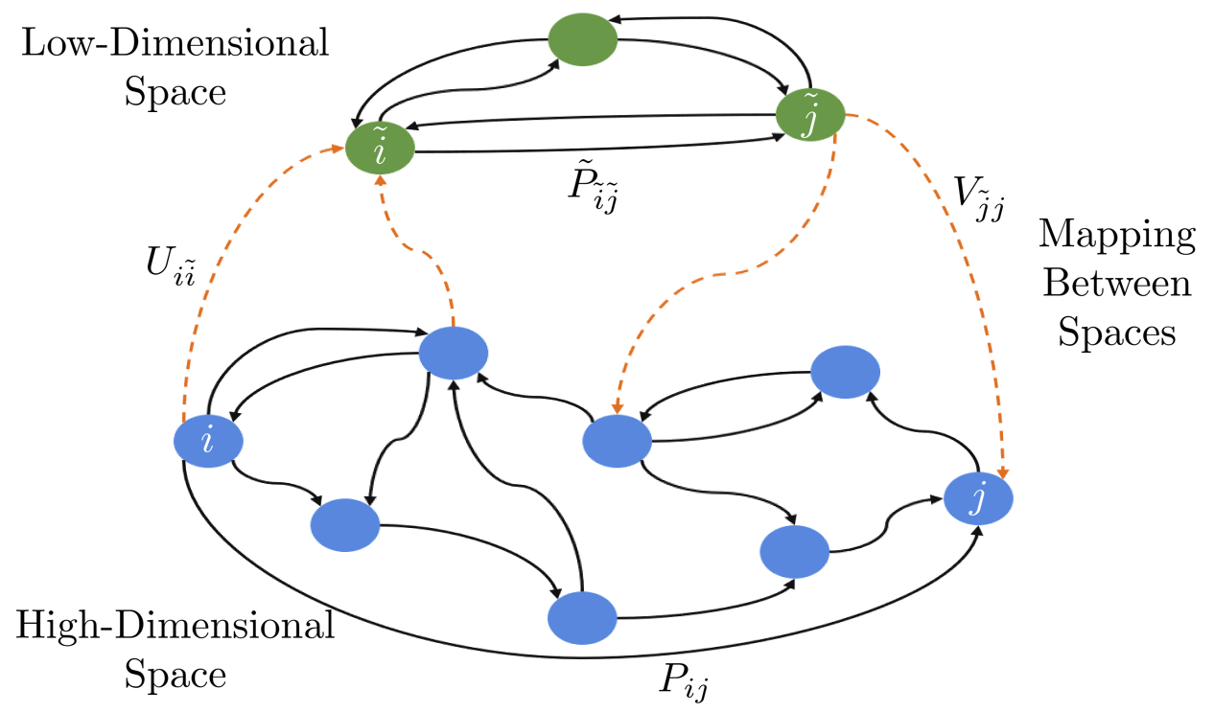}
	\caption{Mapping between high- and low-dimensional spaces. A transition between two states in the original Markov chain ($P_{ij}$) is equivalent to concatenation of the following sequence: a mapping from high-dimensional space to low-dimensional space ($U_{i\tilde{i}}$), a transition in the low-dimensional space ($\tilde{P}_{\tilde{i}\tilde{j}}$), and a mapping from low-dimensional space back to the high-dimensional space ($V_{\tilde{j}j}$).}
	\label{fig:mapping}
\end{figure}

\renewcommand\algorithmicdo{}
\begin{algorithm}[t]
	\caption{Block Coordinate Gradient Descent (BCGD)}
	\label{Alg1}
	\begin{algorithmic}[1]
		\STATE \textbf{Input:} Probability transition matrix $\P$, number of low-dimensional states $k$, step sizes $\alpha$, $\beta$, and $\gamma$, regularization parameters $\lambda_u$ and $\lambda_v$, maximum number of iterations $T$\\
		\STATE \textbf{Output:} Factor matrices $\U$, $\Pt$, and $\V$\\ 
		\STATE \textbf{Initialization:} Initialize $\U_0$ at random\\
		\FOR  { $t=0, 1, 2\dots, T-1$}
		\STATE Update rule for $\U_{t+1}$
		\begin{itemize}
			\item $\gf(\U_t) = -(\P-\U_t\Pt_t\V_t)\V_t^\top\Pt_t^\top$ 
			\item $\U_{t+\frac{1}{2}} = \U_{t}-\alpha_t\gf(\U_t)$
			\item $\U_{t+1} = \Pi_{\triangle_k}\left(\mathcal{T}_{\frac{\lambda_u}{2}}(\U_{t+\frac{1}{2}})\right)$ (Algorithm \ref{Alg2})
		\end{itemize}
		\STATE Update rule for $\Pt_{t+1}$
		\begin{itemize}
			\item $\gf(\Pt_t) = -\U_{t+1}^\top(\P-\U_{t+1}\Pt_t\V_t)\V_t^\top$ 
			\item $\Pt_{t+\frac{1}{2}} = \Pt_{t}-\beta_t\gf(\Pt_t)$
			\item $\Pt_{t+1} = \Pi_{\triangle_k}(\Pt_{t+\frac{1}{2}})$ (Algorithm \ref{Alg2})
		\end{itemize}
		\STATE Update rule for $\V_{t+1}$
		\begin{itemize}
			\item $\gf(\V_t) = -\Pt_{t+1}^\top\U_{t+1}^\top(\P-\U_{t+1}\Pt_{t+1}\V_t)$ 
			\item $\V_{t+\frac{1}{2}} = \V_{t}-\gamma_t\gf(\V_t)$
			\item $\V_{t+1} = \Pi_{\triangle_n}\left(\mathcal{T}_{\frac{\lambda_v}{2}}(\V_{t+\frac{1}{2}})\right)$ (Algorithm \ref{Alg2})
		\end{itemize}
		\ENDFOR
	\end{algorithmic}
\end{algorithm}

The objective function in \eqref{p0} consists of a convex function and hence 
is convex in each of the matrices when the other matrices are fixed. However, due to the fact that the first term in the objective function contains a product of the unknowns, \eqref{p0} is generally a nonconvex program. 
Notice that even if $P$ is perfectly decomposable into the corresponding factors, any permutation of the low-rank abstraction $\Pt$ is also a solution. Therefore, in general, Problem~\ref{p0} has at least $k!$ global optima.

To facilitate a computationally efficient search for the solution of \eqref{p0}, we rely on a modified gradient search algorithm which exploits the special structures of $\U$, $\Pt$, and $\V$. The algorithm (summarized as Algorithm \ref{Alg1}) is essentially a block coordinate gradient descent (BCGD) method that alternatively updates matrices $\U$, $\Pt$, and $\V$ in an iterative fashion starting from an initial point $(\U_0,\Pt_0,\V_0)$. That is, in $(t+1)\ts{st}$ iteration ($t = 0,\dots,T-1$ where $T$ is the total number of iterations), given $(\U_t,\Pt_t,\V_t)$ we optimize with respect to $\U$ to find $\U_{t+1}$. Similarly, we find updated $\Pt_{t+1}$ and $\V_{t+1}$ using the values $(\U_{t+1},\Pt_t,\V_t)$ and $(\U_{t+1},\Pt_{t+1},\V_t)$, respectively. 

In Algorithm \ref{Alg1}, $\mathcal{T}_{\eta}(x) = (|x|-\eta)_+\mathrm{sgn}(x)$ is the so-called
shrinkage-thresholding operator that acts on each element of the given matrix, and $\Pi_{\triangle_d}(.)$ denotes the projection operator that projects each row of its argument onto the probability simplex in $\R^d$. This projection can be efficiently computed by the method of \cite[Algorithm 1]{chen2011projection} that we summarize in Algorithm \ref{Alg2} for completeness.

\renewcommand\algorithmicdo{}
\begin{algorithm}[t]
	\caption{Projection onto $\triangle_d$}
	\label{Alg2}
	\begin{algorithmic}[1]
		\STATE \textbf{Input:} $y = [y_1,\dots,y_d]^\top \in \R^d$\\
		\STATE \textbf{Output:} projection $\Pi_{\triangle_d}(y)$\\ 
		\STATE  Sort $y$ in the ascending order as $y_{(1)}\leq \dots \leq y_{(d)}$\\
		\FOR  { $i=d-1, d-2, \dots, 1$}\vspace{0.1cm}
		\STATE $b_i = \frac{\sum_{j=i+1}^d y_{(j)} - 1}{d-i}$\vspace{0.1cm}
		 \IF {$b_i \geq y_{(i)}$}
		 \STATE $\bar{b} = b_i$
		 \RETURN $\Pi_{\triangle_d}(y) = (y-\bar{b})_+$
		\ENDIF
		\ENDFOR \vspace{0.2cm}
		\STATE $\bar{b} = \frac{\sum_{j=1}^d y_{(j)} - 1}{d}$\vspace{0.1cm}
		\RETURN $\Pi_{\triangle_d}(y) = (y-\bar{b})_+$
	\end{algorithmic}
\end{algorithm}

\begin{figure*}
	\centering
	\begin{subfigure}{0.45\textwidth}
		\centering
		\includegraphics[width=\linewidth]{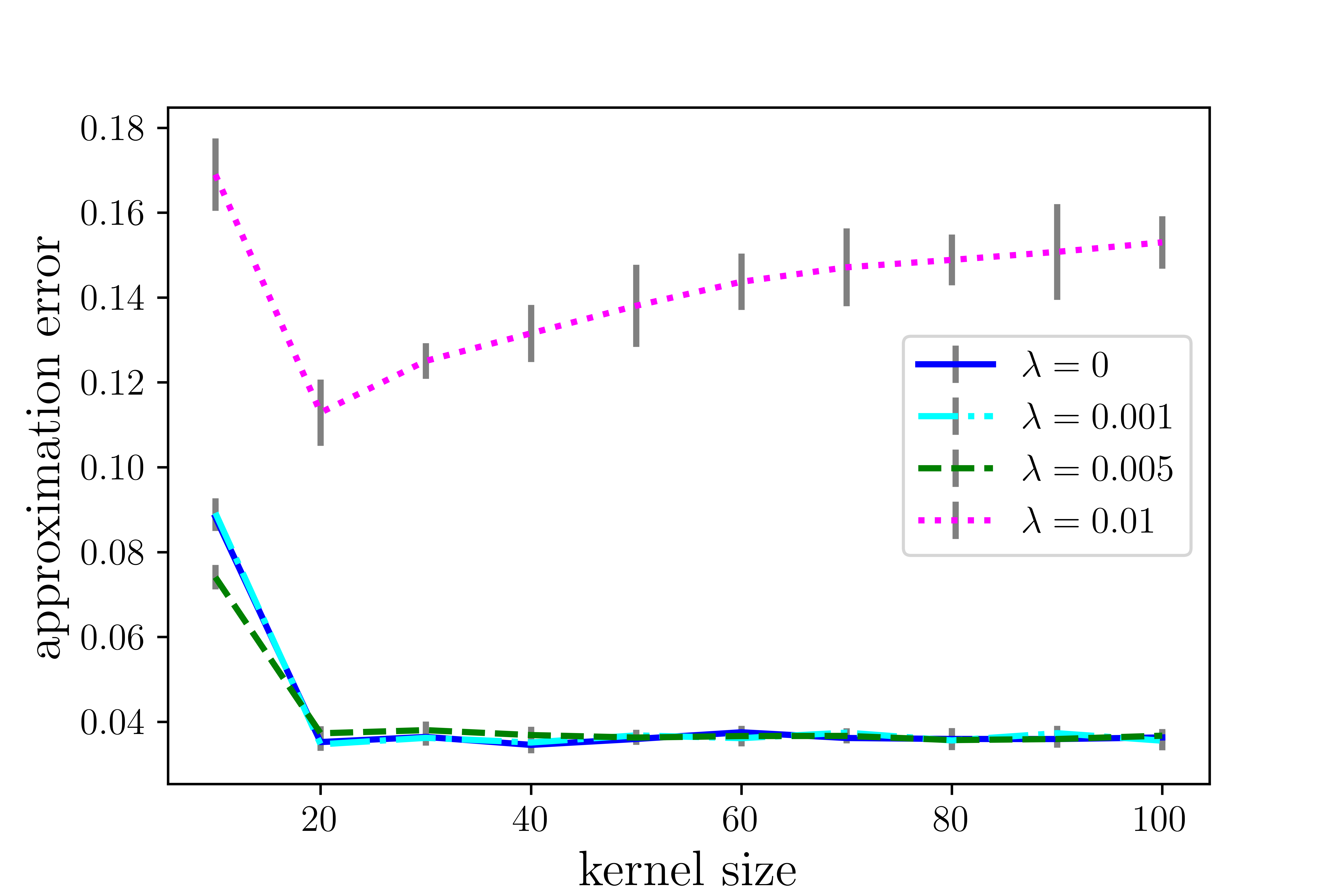}
		\caption{\footnotesize approximation error}
		\label{fig:sub1}
	\end{subfigure}%
	\begin{subfigure}{.45\textwidth}
		\centering
		\includegraphics[width=\linewidth]{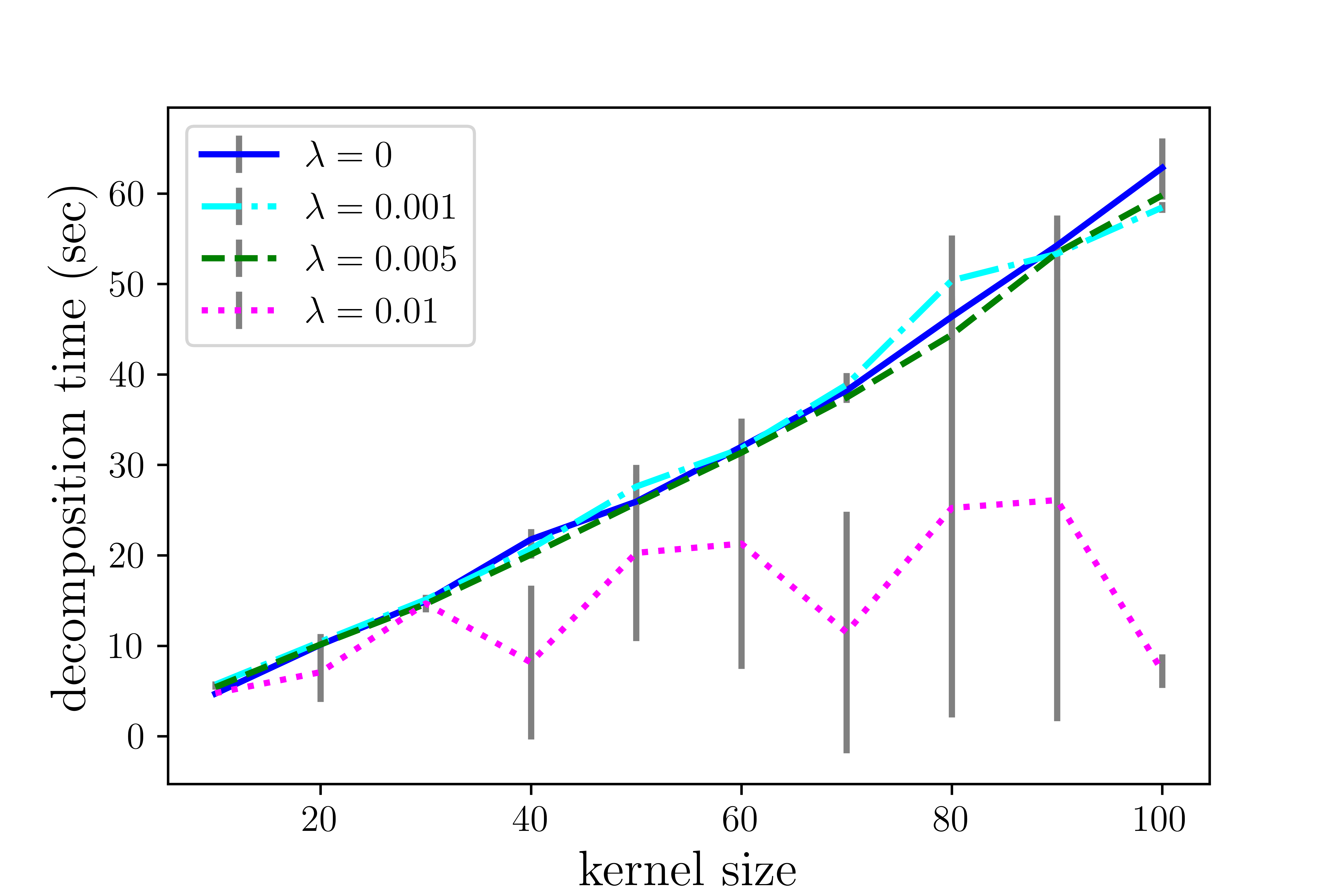}
		\caption{\footnotesize decomposition time}
		\label{fig:sub2}
	\end{subfigure}
	\caption{Effect of kernel size and mapping sparsity on quality of approximation and computational complexity.}
	\label{fig:errortime}
\end{figure*}

\subsection{Convergence Analysis of BCGD}

In this section, we analyze the convergence properties of BCGD. Specifically, in Theorem \ref{thm:c} we
establish that given judicious choices of the step sizes, the value of the objective function in \eqref{p0} decreases as one alternates between updating the factor matrices  which in turn implies that the BCGD algorithm converges to a stationary point of the nonconvex optimization task in \eqref{p0}.
\begin{theorem}\label{thm:c}
Assume the step sizes of Algorithm 1 satisfy
\begin{equation}\label{eq:alpha}
\alpha_t=\frac{C_1\|\gf(\U_t)\|_F^2}{\|\gf(\U_t)\Pt_t\V_t\|_F^2},
\end{equation}
\begin{equation}\label{eq:beta}
\beta_t=\frac{C_2\|\gf(\V_t)\|_F^2}{\|\U_{t+1}\gf(\Pt_t)\V_t\|_F^2},
\end{equation}
\begin{equation}\label{eq:gamma}
\gamma_t=\frac{C_3\|\gf(\Pt_t)\|_F^2}{\|\U_{t+1}\Pt_{t+1}\gf(\V_t)\|_F^2},
\end{equation}
where $C_1,C_2,C_3 \in(0,2)$. Then, the solution $(\U^\ast,\Pt^\ast,\V^\ast)$ found by the BCGD scheme	is a stationary point of \eqref{p0}. 
\end{theorem}
\begin{proof}
Let $f(\U,\Pt,\V)= \frac{1}{2}\|\P-\U\Pt\V \|^{2}_{F}$.
For the proposed algorithm to converge, it must hold that 
\begin{equation}
f(\U_{t+1},\Pt_{t+1},\V_{t+1}) \leq f(\U_{t},\Pt_{t},\V_{t}), 
\end{equation}
for all $t$. Note that since $\mathcal{T}_{\eta}(x)$ and $\Pi_{\triangle_d}$ are projections onto convex sets of constraints (the former being the projection operator onto the $\ell_1$ ball), following a similar analysis as those in the proofs of the projected gradient descent and the iterative shrinkage-thresholding algorithms (ISTA) \cite{beck2009fast,bubeck2015convex} 
\[f(\U_{t+1},\Pt_{t},\V_{t}) \leq f(\U_{t+\frac{1}{2}},\Pt_{t},\V_{t}),\] \[f(\U_{t+1},\Pt_{t+1},\V_{t}) \leq f(\U_{t+1},\Pt_{t+\frac{1}{2}},\V_{t}),\] \[f(\U_{t+1},\Pt_{t+1},\V_{t+1}) \leq f(\U_{t+1},\Pt_{t+1},\V_{t+\frac{1}{2}}).\] Thus, it suffices to show 
\begin{equation}\label{eq:cond1}
f(\U_{t+\frac{1}{2}},\Pt_{t},\V_{t}) \leq f(\U_{t},\Pt_{t},\V_{t}),
\end{equation}
\begin{equation}\label{eq:cond2}
f(\U_{t+1},\Pt_{t+\frac{1}{2}},\V_{t}) \leq f(\U_{t+1},\Pt_{t},\V_{t}),
\end{equation}
\begin{equation}\label{eq:cond3}
f(\U_{t+1},\Pt_{t+1},\V_{t+\frac{1}{2}}) \leq f(\U_{t+1},\Pt_{t+1},\V_{t}).
\end{equation}
We now show that under \eqref{eq:alpha}, the sufficient condition \eqref{eq:cond1} holds. Proof of the parts that \eqref{eq:cond2} and \eqref{eq:cond3} hold under \eqref{eq:beta} and \eqref{eq:gamma}, respectively, follows a similar argument.

Note that
\begin{multline}\label{eq:diff}
f(\U_{t+\frac{1}{2}},\Pt_{t},\V_{t}) - f(\U_{t},\Pt_{t},\V_{t})\\
\quad=\frac{1}{2}\|\P-\U_{t}\Pt\V_{t} + \alpha_t\gf(\U_t)\Pt_{t}\V_{t}\|^{2}_{F}\\
\phantom{111111111111111}-\frac{1}{2}\|\P-\U_{t}\Pt\V_{t} \V_t\|^{2}_{F}\\
\quad= \alpha_t\mathrm{Tr}\left((\P-\U_{t}\Pt\V_{t})^\top\gf(\U_t)\Pt_{t}\V_{t}\right)\\
+\frac{\alpha_t^2}{2}\|\gf(\U_t)\Pt_{t}\V_{t}\|^{2}_{F}.
\end{multline}
Now, consider the first term in the last line of \eqref{eq:diff}. Following straightforward linear algebra, we obtain
\begin{equation}
\begin{aligned} 
&\mathrm{Tr}\left((\P-\U_{t}\Pt\V_{t})^\top\gf(\U_t)\Pt_{t}\V_{t}\right)\\
&=-\mathrm{Tr}\left((\P-\U_{t}\Pt\V_{t})^\top(\P-\U_t\Pt_t\V_t)\V_t^\top\Pt_t^\top\Pt_{t}\V_{t}\right)\\
&=-\mathrm{Tr}\left(\Pt_{t}\V_{t}(\P-\U_{t}\Pt\V_{t})^\top(\P-\U_t\Pt_t\V_t)\V_t^\top\Pt_t^\top\right)\\
&=-\|(\P-\U_t\Pt_t\V_t)\V_t^\top\Pt_t^\top\|_F^2=-\|\gf(\U_t)^\top\|_F^2.
\end{aligned} 
\end{equation}
Therefore,
\begin{multline}
f(\U_{t+\frac{1}{2}},\Pt_{t},\V_{t}) - f(\U_{t},\Pt_{t},\V_{t})\\
\qquad=\frac{\alpha_t^2}{2}\|\gf(\U_t)\Pt_{t}\V_{t}\|^{2}_{F}
-\alpha_t\|\gf(\U_t)^\top\|_F^2\\
\hspace{-2cm}=\left(\frac{C_1^2}{2}-C_1\right)\frac{\|\gf(\U_t)^\top\|_F^4}{\|\gf(\U_t)\Pt_{t}\V_{t}\|_F^2},
\end{multline}
where the last equality follows according to the definition of $\alpha_t$ in \eqref{eq:alpha}. It is now clear that if $C_1 \in (0,2)$ it must be the case that \eqref{eq:cond1} holds, which in turn implies convergence of Algorithm 1.
\end{proof}

\subsection{Computational Complexity of BCGD}

The computational complexity of the proposed BCGD algorithm is analyzed next. Note that the determining factor for cost per iteration of Algorithm \ref{Alg1} is computation of the gradients. Finding $\gf(\U_t)$ incurs $\mathcal{O}(nk)$ as it contains matrix products between $k\times k$ and $k\times n$ matrices. Similarly, $\gf(\Pt_t)$ and $\gf(\V_t)$ require $\mathcal{O}(nk)$ computational costs. Thus, Algorithm \ref{Alg1} incurs a linear complexity of $\mathcal{O}(nkT)$.
\section{Simulation Results}\label{sec:sim}
We implemented the proposed BCGD algorithm for nonnegative matrix factorization in Python.\footnote{The code is available at https://github.com/MahsaGhasemi/state-abstraction}
We evaluated the performance of the proposed abstraction solution for a variety of parameters. In particular, we investigated the effect of sparsity-promoting term on the approximation quality of the factorized transition matrix. We also ran the algorithm for different values of step size and compared the convergence of BCGD. 
For these simulations, we generated a transition matrix $\P$ of size $100 \times 100$ for the original Markov chain. The transition matrix has a rank of $25$ and is constructed by multiplying three stochastic matrices. Each stochastic matrix is generated by independently sampling its rows by a uniform sampling from the simplex of proper size. We set the number of iterations of BCGD to $1000$. If the difference between two consecutive instances of a factor, i.e., $\U_t$, $\Pt_t$, or $\V_t$ falls below a threshold $10^{-8}$ of their magnitude, the algorithm terminates, where the magnitudes are measured by Frobenius norm. Similarly, if the difference between two consecutive values of the objective function falls below $10^{-8}$, the algorithm terminates. We run the simulation with each set of parameters for $10$ independent instances and report the average values along the standard deviations.
All simulations were run on a machine with 2.0 GHz Intel Core i7-4510U CPU and with 8.00 GB RAM.

\subsection{Effect of Regularization Parameter on Performance}\label{ssec:sim-lambda}

One of the key differences of the proposed abstraction formulation is the integration of a $\ell_1$-norm regularization term in the optimization objective. This term promotes sparsity for the bidirectional mapping between the original and the kernel space. Intuitively, this sparsity ensures that the meta states in the kernel space are representative of a small number of original states. The sparsity level of the mappings depends on regularization parameters $\lambda_u$ and $\lambda_v$. 

In Figure~\ref{fig:errortime}, we demonstrate the effect of these parameters on the quality of the solutions. We compare four different values, more specifically, $\lambda = \lambda_u = \lambda_v \in \{0,0.001,0.005,0.01\}$. The results in Figure~\ref{fig:errortime}(a) show that a careful selection of $\lambda$, while increasing sparsity, does not affect the quality of the decomposition in terms of the approximation error, computed by $\|\P-\U\Pt\V\|^{2}_{F}$. However, a large $\lambda$ leads to high approximation error. Therefore, one has to find the right trade-off between lower approximation error and higher sparsity of the mappings. Furthermore, we can see that error significantly reduces once the size of the kernel transition is set to values higher than $20$. Above that value, the approximation error only slightly changes. Therefore, the proposed algorithm is successful in identifying a low-rank representation.

Figure~\ref{fig:errortime}(b) depicts the running time of the BCGD algorithm. As derived in Section~IV-B, the running time is linear with respect to the kernel size and the results reflect that. Furthermore, the addition of $\ell_1$-norm regularization term has negligible effect on the running time. Note that for $\lambda = 0.01$, the algorithm terminates early as it cannot improve the objective value sufficiently. Furthermore, for this choice of $\lambda$, we can observe high variance in the running time due to the varying sparsity degree of the randomly generated matrices.

\subsection{Effect of Step Size on Convergence}\label{ssec:sim-conv}

In Section~IV-A, we derived necessary conditions on the step sizes $\alpha$, $\beta$, and $\gamma$ for the convergence of the BCGD algorithm. In this section, we show the sensitivity of the convergence to different values of the step size. To that end, we ran the algorithm for different values of step size that we kept constant throughout the run. In particular, we ran the algorithm for $\alpha = \beta = \gamma \in \{0.002,0.02,0.2\}$. Figure~\ref{fig:stepsize} depicts the evolution of error over the course of $500$ iterations of BCGD. While the algorithm converges for all three values, the smaller step sizes achieve lower approximation error at the end. We also observed that the algorithm would often diverge for step sizes over $0.2$.

\begin{figure}[t]
    \centering
    \includegraphics[width=0.45\textwidth]{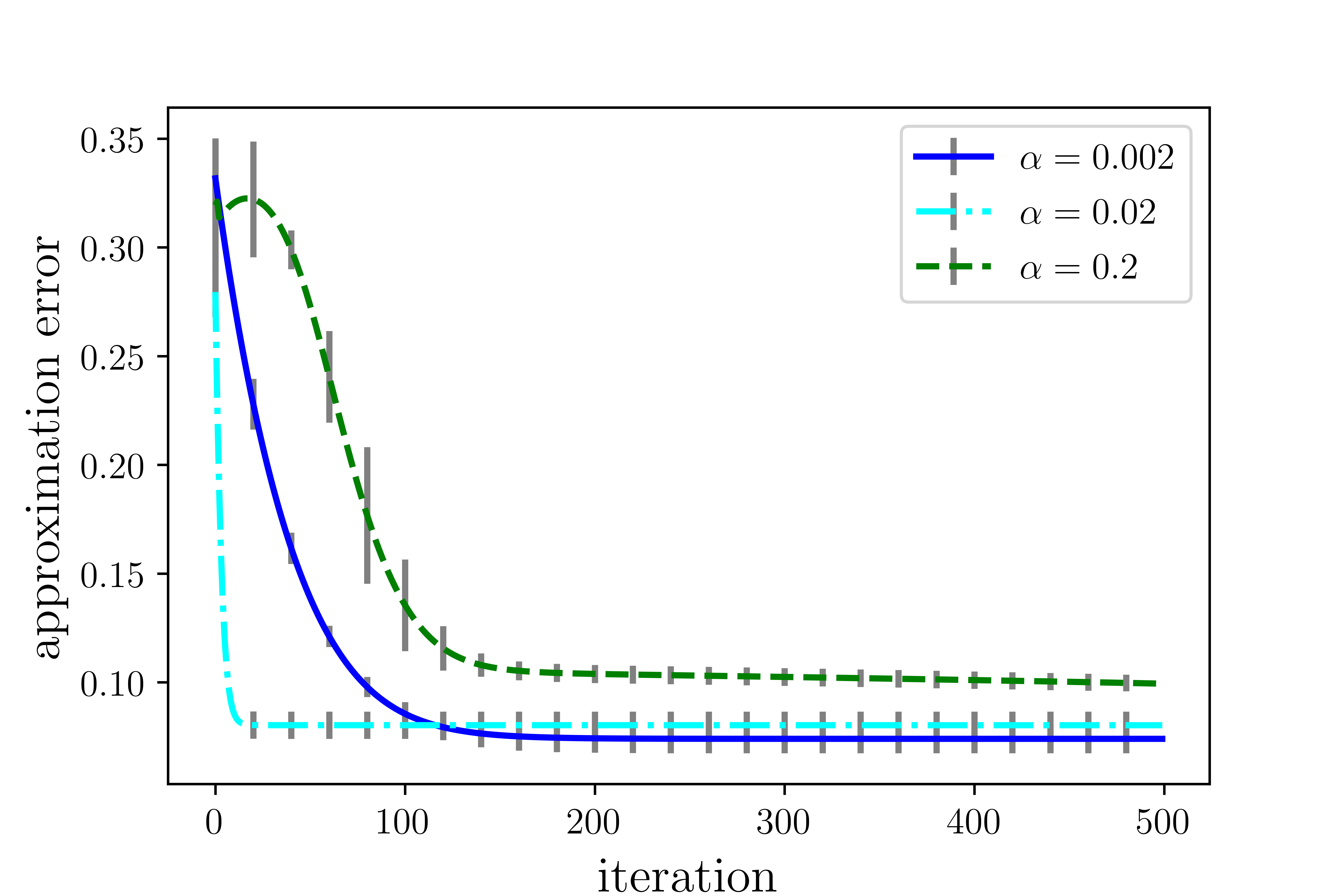}
    \caption{The effect of step size on the convergence of BCGD.}
    \label{fig:stepsize}
\end{figure}
\section{Conclusion} \label{sec:concl}
We studied the problem of approximating a large-scale Markov chain by a surrogate model in a low-dimensional state space, called kernel space. We proposed a nonnegative matrix factorization formulation that learns a low-rank kernel transition as well as a pair of forward and backward mappings while promoting a sparse connection between the high and low-dimensional states. We showed that the formulated optimization is amenable to an efficient iterative solution that converges to a stationary solution under a judicious schedule of step sizes. 
As part of future work, we aim to extend the proposed matrix factorization formulation to model reduction of Markov decision processes. Furthermore, we would like to evaluate the abstracted Markov decision process in different analyses, including model checking and value function approximation.

\bibliographystyle{ieeetr}
\bibliography{main.bib}

\end{document}